\pdfoutput=1

\documentclass{article}

\usepackage{microtype}
\usepackage{graphicx}
\usepackage{subcaption}
\usepackage{booktabs}
\usepackage{times}
\usepackage[numbers]{natbib}
\usepackage{authblk}

\usepackage{hyperref}

\usepackage{amsmath,amsfonts,bm}

\def\eqref#1{equation~\ref{#1}}

\def\1{\bm{1}}

\DeclareMathAlphabet{\mathsfit}{\encodingdefault}{\sfdefault}{m}{sl}
\SetMathAlphabet{\mathsfit}{bold}{\encodingdefault}{\sfdefault}{bx}{n}

\newcommand{\E}{\mathbb{E}}

\newcommand{\R}{\mathbb{R}}

\usepackage{amsmath,amsthm,amsfonts,amssymb}
\usepackage{mathtools}
\usepackage{fullpage}
\usepackage{geometry}
\usepackage{algorithm}
\usepackage{algorithmic}
\usepackage{float}

\usepackage[capitalize,noabbrev]{cleveref}

\usepackage{xcolor}
\usepackage{pifont}
\usepackage{enumitem}
\setlist[enumerate]{nosep}

\hypersetup{colorlinks=true, linkcolor=blue, citecolor=blue, urlcolor=blue}

\newcommand{\cmark}{\textcolor{green!70!black}{\ding{51}}}
\newcommand{\xmark}{\textcolor{red!80!black}{\ding{55}}}
\newcommand{\pmark}{\textcolor{orange!80!black}{$\boldsymbol{\sim}$}}  

\theoremstyle{plain}
\newtheorem{theorem}{Theorem}[section]
\newtheorem{proposition}[theorem]{Proposition}
\newtheorem{lemma}[theorem]{Lemma}
\newtheorem{corollary}[theorem]{Corollary}
\theoremstyle{definition}
\newtheorem{definition}[theorem]{Definition}

\theoremstyle{remark}

\usepackage{tcolorbox}
\tcbuselibrary{skins,breakable}

\newtcolorbox{calloutbox}{
  colback=gray!05,
  colframe=gray!10,
  coltitle=black,
  boxrule=0pt,
  title=\textit{Not} Worst Case Analysis,
  fonttitle=\bfseries,
  boxsep=2pt,
  left=2pt,
  right=2pt,
  top=2pt,
  bottom=2pt
}

\newtcolorbox{scopebox}{
  colback=gray!05,
  colframe=gray!10,
  coltitle=black,
  boxrule=0pt,
  title=Scope of the A-BB guarantee,
  fonttitle=\bfseries,
  boxsep=2pt,
  left=2pt,
  right=2pt,
  top=2pt,
  bottom=2pt
}

\usepackage{thm-restate}

\title{Towards Provably Unbiased LLM Judges via Bias-Bounded Evaluation}

\author[1,2]{Benjamin Feuer}
\author[3]{Lucas Rosenblatt}
\author[2]{Oussama Elachqar}
\affil[1]{Stanford University}
\affil[2]{Oumi.AI}
\affil[3]{New York University}
\date{}

\begin{document}

\maketitle

\begin{abstract}
As AI models progress beyond simple chatbots into more complex workflows, we draw ever closer to the event horizon beyond which AI systems will be utilized in autonomous, self-maintaining feedback loops. Any autonomous AI system will depend on automated, verifiable rewards and feedback; in settings where ground truth is sparse or non-deterministic, one practical source of such rewards is an LLM-as-a-Judge. Although LLM judges continue to improve, the literature has yet to introduce systems capable of enforcing standards with strong guarantees, particularly when bias vectors are unknown or adversarially discovered. To remedy this issue, we propose \textit{average bias-boundedness} (A-BB), an algorithmic framework which formally guarantees reductions of harm/impact as a result of any measurable bias in an LLM judge. Evaluating on Arena-Hard-Auto with four LLM judges, we achieve ($\tau=0.5$, $\delta=0.01$)-bias-bounded guarantees while retaining 61–99\% correlation with original rankings across formatting and schematic bias settings, with most judge-bias combinations exceeding 80\%. The code to reproduce our findings is available at \url{https://github.com/penfever/bias-bounded-evaluation}.
\end{abstract}

\begin{figure*}[htbp]
\centering
\includegraphics[width=0.9\textwidth]{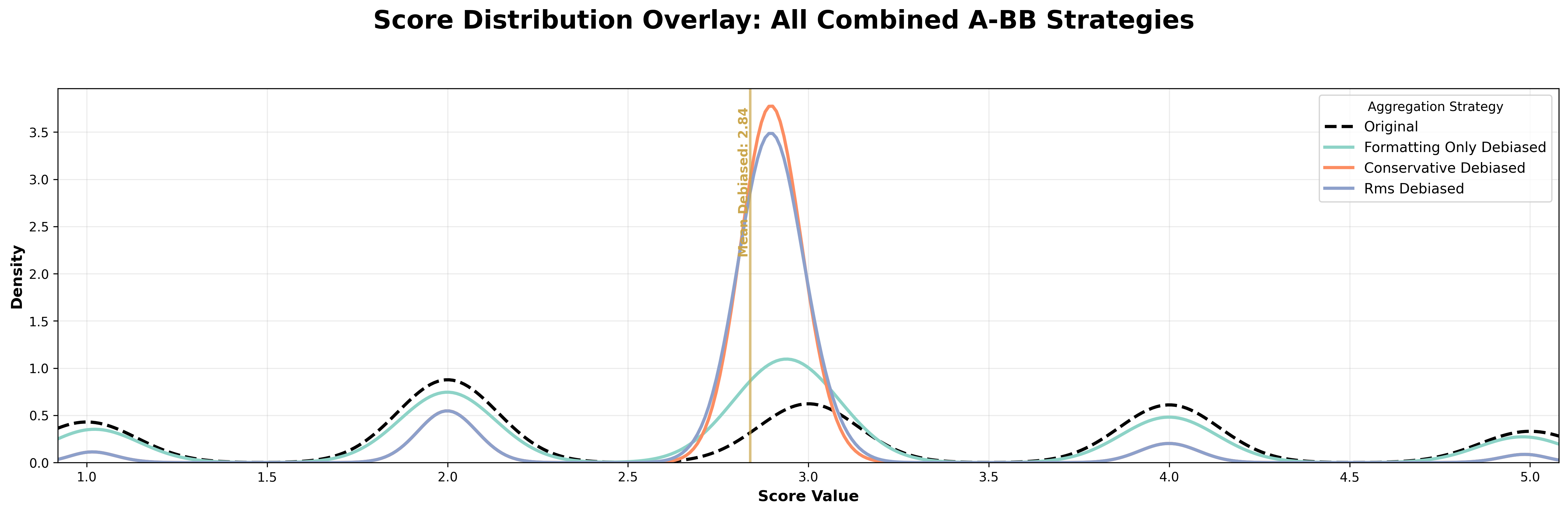}
\caption{\textbf{Bias-bounded evaluation constrains the impact of harms in judge scoring.} This before-and-after visualization of the score distributions from an LLM judge (Likert-scale) on the popular Arena-Hard-Auto benchmark shows how true uncertainty can be captured via a compacted score distribution. After the average bias-boundedness (A-BB) algorithm is applied, the original integer-valued score are transformed into a debiased, continuous  trajectory which accurately represents the measured uncertainty of the evaluation. The plot shows a KDE density map of the score distribution before and after transformation, with a conservative $\tau = 0.5, \delta = 0.03, \texttt{dim}=500$, averaged across a panel of four judges.}
\label{fig:combined_distribution}
\end{figure*}

\section{Introduction}

Agentic workflows are an increasingly popular method of deploying large language models (LLMs) in complex, real-world settings; it appears increasingly likely that these agentic systems will soon be used in autonomous, self-maintaining feedback loops \citep{plaat2025agentic}. This very year, for instance, the 2026 AAAI Conference added an LLM judge (AI-Powered Peer Review System) to the panel of reviewers for its scientific submissions, with mixed results~\cite{aaai2025llm}. In autonomous deployments, the risk of harm is greatly magnified; for instance, the well-known company Replit.AI recently suffered an incident in which a client company database (with thousands of entries) was accidentally deleted by an AI agent~\citep{tyson2025replit}.

AI agents will not be safe to deploy autonomously in most contexts until we have verifiable rewards and feedback. In settings where ground truth is sparse or difficult to parse, the most practical source for such feedback is an LLM-as-a-Judge \citep{zhugeagent}; by instantiating a (potentially airgapped) judgment stage into the  agentic feedback loop, it should, in principle, be possible to detect and mitigate failures in real time.

Unfortunately, recent works show that LLM judges, in naive deployments, present a range of distinct failure modes~\citep{feuer2024styleoutweighs,wang2024justice,chen2024preference,pezeshkpour2024trust,li2024limits}. Worse still, it remains unclear to what degree these effects are cumulative when acting in concert. Last but not least, when it comes to bias, we cannot help to "know what we don't know" -- there will always remain bias sources, including adversarially discovered sources, that we are capable of \textit{measuring}, but not \textit{explaining}.

In order to mitigate these issues, we propose a novel framework: \textbf{bias-bounded evaluation (BBE)} in \cref{sec:bbe}. At a high level, BBE measures how sensitive an LLM judge is to biases and then \textit{renders that sensitivity in the judgment scores}. Biases can be very simple (e.g. varying presentation order or format) or arbitrarily complex; the only limitation in this framework is that a bias must be subject to a valid and reliable measure, in the units of judgment, in order to be targeted. BBE injects calibrated Gaussian noise to the judge's scores to \textit{mitigate} the effects of the bias, using a novel mechanism we call \textbf{average bias-boundedness (A-BB)}, described in detail in \cref{sec:abb}. An attentive reader may notice that our mechanism takes inspiration from techniques in differential privacy~\citep{dwork2006calibrating}, though our analysis differs substantially, as does our goal.

\paragraph{Contributions} Concretely, for any fixed judgment space, bias space and rubric factor, we describe a mechanism that can, with high probability, formally bound the likelihood that the average-case bias will exceed a particular quantity. In particular, we consider our contributions as follows:
\begin{enumerate}
    \item We propose \textbf{bias-bounded evaluation}, an algorithmic framework which formally guarantees reductions of harm/impact as a result of any measurable biases in an LLM judge, even when the causes of those biases are complex, intersecting or unknown.
    \item We demonstrate empirically that BBE can retain signal and provide guarantees in realistic settings with large amounts of bias.
    \item We release our codebase with a working implementation suitable for future development.\footnote{Code available at \url{https://github.com/penfever/bias-bounded-evaluation}.}
\end{enumerate}

\section{Bias-Bounded Evaluation}
\label{sec:bbe}

Our assumptions in this analysis are light; we assume access to some real-valued Judgment Space, and that an ideal judgment, for a particular judge, exists for that space. The ideal judgment is not necessarily correct; it simply represents the set of judgments the judge would have rendered, conditioned on the inputs, in the absence of unwanted bias. As we have discussed, bias in LLM-based evaluations can arise in many subtle ways, from formatting to presentation order to some other latent, unavoidable rubric factors. These sources of bias introduce systematic deviations from what an ideal judgment would be. Inspired by the idea of leveraging \textit{noise} to bound \textit{sensitivity} to perturbations of an output vector (as is standard with worst-case guarantees in differential privacy \cite{dwork2006calibrating}), we introduce \textbf{bias-bounded evaluation (BBE)}. Bias-bounded evaluation measures how sensitive a judge is to contextual perturbations, and then injects calibrated Gaussian noise to mitigate the effects of the bias. We make the bias \textit{and} the noise magnitude quantifiable, so that BBE surfaces and mitigates epistemic uncertainty in automated evaluation in a controlled manner.

\subsection{Basic Definitions}

We begin by establishing the mathematical framework for our analysis.

\begin{definition}[Judgment Space]
Let $\mathcal{J}$ be the space of all possible judgments, where judgment $j \in \mathcal{J}$ is a vector of scores: $j = (s_1, s_2, \ldots, s_k, s_{overall})$ and $s_i$ represents the score for factor $i$ and $s_{overall}$ is the overall judgment score.
\end{definition}

\begin{definition}[Rubric Factors]
Let $\mathcal{R} = \{r_1, r_2, \ldots, r_k\}$ be the set of explicit rubric factors, where each $r_i: \mathcal{P} \rightarrow \mathbb{R}$ maps a prompt-response pair to a real-valued score.
\end{definition}

\begin{definition}[Bias Space]
Let $\mathcal{B}$ be the space of all possible bias functions, where each $b \in \mathcal{B}$ represents a systematic deviation from ideal judgment based on factors not captured in $\mathcal{R}$.
\end{definition}

\subsection{Judgment Context and Neighboring Contexts}

\begin{definition}[Judgment Context]
A judgment context $D$ is a dataset that contains: (1.) A set of prompt-response pairs $\{(p_i, r_i)\}_{i=1}^n$, (2.) An associated ground truth labels or reference judgments (if available), and (3.) Environmental factors that may influence judgment (presentation order, formatting, etc.)
\end{definition}

\begin{definition}[Neighboring Judgment Contexts]
Two judgment contexts $D$ and $D'$ are neighboring (denoted $D \sim D'$) if they differ in exactly one prompt-response pair, where that pair has been modified by a bias-introducing perturbation that preserves germane semantic content (i.e., content pertaining to the rubric criteria). Examples include: (i) \textit{formatting perturbations} such as reformatting or stylistic paraphrasing; (ii) \textit{schematic perturbations} that vary emphasis on rubric factors without changing the response; and (iii) \textit{agreeableness perturbations} that introduce subtle errors to test whether the judge detects flaws.
\end{definition}

\section{Average Bias-Bounded Evaluation (A-BB)}
\label{sec:abb}

We consider a mechanism that maps a \textit{fixed} judgment context to a score vector. This can be a highly sensitive mapping, and even small changes in this input context can shift output scores; this ``implicit bias'' is what we'd like to control.

\begin{calloutbox}
Under a \textit{worst-case} style analysis (as is standard with e.g. global differential privacy and many of its variants), one tries to bound the effect of \textit{any} perturbation (across \textit{any} neighboring dataset) \citep{dwork2014algorithmic, cummings2024advancing}. This is inherently too conservative in the setting of LLM judgements and contexts, as it tunes the mechanism to the worst possible change (which could be on the order of the full judgement scale), and forces one to consider \textit{any} two judgement contexts (that can be adversarially chosen).
\end{calloutbox}


Instead, we consider an \textbf{\textit{average-case} bias-bounded guarantee}. Also, instead of considering \textit{any} two judgement contexts, we allow ourselves to \textit{fix} a judgement context, as we generally are assumed to control both the judge and application setting and thus do not have to be sensitive to a potentially adversarial selection of judgement context.

\paragraph{Defining A-BB}
For \textit{average case} bias boundedness, we \textit{fix} a judgement context $D$.\footnote{This is akin to differentially private \textit{local} sensitivity \citep{nissim2007smooth, dwork2009differential}.} Then, we instantiate a function $T$ to be a randomized \textit{neighbor generator}. Let $T: \mathcal{D} \to \mathcal{D}$. Let $D' \underset{T}{\sim} D$ be a neighboring dataset resulting from a random draw from the generator, $D' = T(D)$. The randomness here could come from any measurable source(s) of bias. Furthermore, we consider a root-mean squared error, which is convenient due to positivity (a similar choice is often made when doing $\ell_2$ smooth-sensitivity analysis in differential privacy\footnote{Additionally, we note that if $T$ truly sampled uniformly from the Hamming-1 token neighbors of the judgement context, then the root-mean-squared sensitivity would be exactly the \textit{average-case} analogue of the local sensitivity (used by smooth sensitivity).}).
\begin{definition}[Root-mean-squared sensitivity] \label{def:rms}
For a deterministic judge $f: \mathcal{D} \to \mathbb{R}^d$ and fixed context $D$, we define,
\begin{align*}
\Delta^{*}_2(f,D) = \left( \mathbb{E}_{D' \underset{T}{\sim} D} [ \|f(D) - f(D')\|_2^2 ] \right)^{1/2}~.
\end{align*}
\end{definition}

\begin{definition}[Average bias boundedness (A-BB)] \label{def:a-bb}
For a fixed judgement context $D$, neighbor generator $T$, tolerance $\tau > 0$, and failure probability $\delta \in (0,1)$, we say a randomized mechanism $M: \mathcal{D} \mapsto \mathbb{R}^d$ is \textit{$(\tau, \delta)$-average bias-bounded} if,
\begin{align*}
\Pr[ \|M(D) - M(D')\|_2 > \tau ] \leq \delta,
\end{align*}
where the probability is taken over both the randomness of the neighbor generator $T$ (i.e., the draw $D' \underset{T}{\sim} D$) and the internal randomness of mechanism $M$.
\end{definition}
In other words, we will seek to control the \textit{tail probability} that a random perturbation (drawn from neighbor generator $T$) together with the randomness that comes from the noise mechanism $M$ causes more than a $\tau$ change in the output.

\begin{scopebox}
We stress that the A-BB certificate stated above is \textit{local} to the fixed judgment context $D$ and the specific neighbor generator $T$ chosen for certification (as well as any unknown generators whose biasing effects are no larger than $T$); it is \textit{not} a distributional generalization guarantee over unseen contexts. For example, later in our experiments, $D$ will correspond to the batch of Arena-Hard-Auto comparisons judged by a fixed judge, and $T$ will sample from allowed contextual perturbations within that batch (see \Cref{sec:experiments} for exact details).
\end{scopebox}

We now give the mechanism and analysis, along with the noise addition stated in algorithmic form in \Cref{alg:gauss_mech}. Note that in what follows we assume $f$ is deterministic and that the mechanism noises $Z,~Z'$ are independent of each other and of $D' \underset{T}{\sim} D$ (i.e., any base-judge randomness is either folded into $f$ or modeled via $Z$).

\begin{restatable}[Gaussian mech. for A-BB: a baseline split bound]{theorem}{GaussMechABB}
\label{thm:gauss_mech_abb}
Consider a judgement context $D$ and a neighbor generator $T$. Let $\Delta := f(D) - f(D')$ where $D' \underset{T}{\sim} D$. Let $M_\sigma(D) = f(D) + Z$ with $Z \sim \mathcal{N}(0, \sigma^2 I_d)$, and further let $B := Z - Z'$, where we set $Z'$ as an independent copy of $Z$, so $B \sim \mathcal N(0, 2\sigma^2 I_d)$. Then for any $\delta_B \in (0,1)$ and any threshold $a \in (0, \tau)$, if
\begin{align*}
\sqrt{2 \sigma^2 \left(d + 2\sqrt{d \log(1 / \delta_B)} + 2\log(1 / \delta_B) \right)} \leq \tau - a~,
\end{align*}
or, equivalently, for any $\sigma$ in the admissible interval $0 < \sigma \leq \sigma_{\max}$ where,
\begin{align*}
\sigma_{\max} := \frac{\tau - a}{\sqrt{2 \left(d + 2\sqrt{d \log(1 / \delta_B)} + 2\log(1 / \delta_B) \right)}}~,
\end{align*}
$M_\sigma(D)$ satisfies $(\tau, \delta)$-average bias boundedness (\Cref{def:a-bb}) with
\begin{align*}
\delta = \delta_B + \frac{\Delta_2^{*} (f, D)^2}{a^2}~.
\end{align*}
\end{restatable}
We defer the proof of \Cref{thm:gauss_mech_abb} to \Cref{sec:deferred_proofs}.

Setting $a = \Delta_2^*(f,D) / \sqrt{\delta_\Delta}$ yields an explicit split, which we state below in \Cref{cor:split_failure_budget} (this is a direct application of \Cref{thm:gauss_mech_abb}, and again we defer the short proof to \Cref{sec:deferred_proofs}).

\begin{restatable}[Splitting the failure budget]{corollary}{SplitFailureBudget}
\label{cor:split_failure_budget}
Fix $\tau > 0$ and a split $\delta = \delta_B + \delta_\Delta$ with $\delta_B, \delta_\Delta \in (0,1)$.
Assume $\tau > \Delta^*_2(f,D)/\sqrt{\delta_\Delta}$. Then the Gaussian mechanism
$M_\sigma(D) = f(D) + Z$, $Z \sim \mathcal N(0, \sigma^2 I_d)$, satisfies $(\tau, \delta)$-A-BB
for any $0 < \sigma \leq \sigma_{\max}$ where,
\begin{align}
\sigma_{\max} = \frac{\tau - \Delta^*_2(f,D) / \sqrt{\delta_\Delta}}
{\sqrt{2 \left(d + 2 \sqrt{d \log(1 / \delta_B)} + 2\log(1 / \delta_B)\right)}}~.
\end{align}
\end{restatable}

\begin{corollary}[Symmetric split]
If we let $\delta_B = \delta_\Delta = \delta / 2$ in~\Cref{cor:split_failure_budget}, we get,
\begin{align*}
\sigma \leq \frac{ \tau - \Delta_2^{*}(f, D)\sqrt{2 / \delta} }{\sqrt{2 \left(d + 2\sqrt{d \log(2 / \delta)} + 2\log(2 / \delta)\right)}}~,
\end{align*}
so long as $\tau > \Delta_2^{*}(f, D)\sqrt{2 / \delta}$.
\end{corollary}

In \Cref{rem:splitting_helps}, we consider how by setting $\Delta = a$, the split calibration tightens the Gaussian radius relative to a no-split bound at the same $\tau$. In \Cref{rem:utility}, for convenience we give the standard 2-norm utility expectation result of adding $d$-dimensional Gaussian noise in this manner. Short proofs again deferred to \Cref{sec:deferred_proofs}.

\begin{restatable}[Why splitting helps]{remark}{WhySplittingHelps}
\label{rem:splitting_helps}
Note that, for any $\tau > 0$ and any $\Delta \in [0, \tau)$, we have that,
\begin{align*}
\frac{\tau - \Delta}{\sqrt{\tau^2 - \Delta^2}} = \sqrt{\frac{\tau - \Delta}{\tau + \Delta}} < 1~.
\end{align*}
\end{restatable}

\begin{restatable}[Utility]{remark}{Utility} \label{rem:utility}
The following is a standard fact of adding Gaussian noise in $d$ dimensions,
\begin{align*}
\mathbb{E}\| M_\sigma(D) - f(D) \|_2^{2} = d \sigma^{2}~.
\end{align*}
\end{restatable}

\textbf{Overall}, in the above average case analysis, we try to shrink $\sigma$ (the variance of the Gaussian noise we add) analytically \textit{as much as possible} such that we can still get a meaningful guarantee on the tail probability for the A-BB statement. For convenience, we also state the procedure in \Cref{alg:gauss_mech} below, with the correct splitting and $\sigma$ parameter; this is the exact procedure we use in our experimentation.

\begin{algorithm}[tb]
\caption{A-BB Gaussian Mechanism} \label{alg:gauss_mech}
\begin{algorithmic}[1]
\REQUIRE Base judge $f$, judgment context $D$, neighbor generator $T$, params $(\tau, \delta)$
\ENSURE $(\tau,\delta)$-average bias-bounded judgment $j'$
\STATE First compute raw judgment $j = f(D)$
\STATE Then, estimate root-mean-squared sensitivity, by
\STATE \quad Sampling $m$ neighbors $D'_1, \dots, D'_m \underset{T}{\sim} D$
\STATE \quad To assign $\Delta_2^{*}(f,D) \leftarrow ( \frac{1}{m} \sum_{i=1}^m \|f(D) - f(D'_i)\|_2^2 )^{1/2}$
\STATE Then, split failure budget as $\delta_B = \delta / 2$, $\delta_\Delta = \delta / 2$
\STATE Ensure that we meet requirement of $\tau > \Delta_2^{*}(f,D) \sqrt{2 / \delta}$
\STATE Then, compute the maximal admissible noise parameter, or
\STATE \quad $\sigma \leftarrow \frac{\tau - \Delta_2^{*}(f,D) \sqrt{2 / \delta}}{\sqrt{2 \left(d + 2\sqrt{d \log(2 / \delta)} + 2\log(2 / \delta)\right)}}$
\STATE Choose any $\sigma \in (0, \sigma_{\max}]$ (by default, $\sigma \gets \sigma_{\max}$ for a tight certificate; smaller $\sigma$ strictly improves utility)
\STATE Finally, sample noise $Z \sim \mathcal{N}(0, \sigma^2 I_d)$,
\STATE And return noisy judgment $j' \leftarrow j + Z$.
\end{algorithmic}
\end{algorithm}

\paragraph{Lipschitz shrinkage of data} One algorithmic adjustment we make, which is heuristically based on the distributions of judgement scores we observe, is a deterministic Lipschitz shrinkage of the judgement score data around a center point. This shrinkage dampens how much scores can move under small contextual perturbations, so less Gaussian noise is needed to achieve the same $(\tau,\delta)$ guarantee. In general, an $L$-Lipschitz mapping is a standard approach to contracting all neighbor-to-neighbor score differences by at most $L$ in a dataset \cite{boyd2004convex}; here we do this to tighten the A-BB certificate with minimal impact on data point order. More formally, before adding Gaussian noise in \Cref{alg:gauss_mech}, we deterministically post-process scores via any $L$-Lipschitz map $g:\R^d \to \R^d$ (with e.g., affine shrinkage $g(x) = \alpha x + (1 - \alpha) \mu$ with $L = \alpha$). This targets the RMS sensitivity (\Cref{def:rms}): $\Delta_2^{*}(g \circ f, D) \leq L \Delta_2^{*}(f,D)$ (\Cref{lem:lipschitz}). Calibrating $\sigma$ using the same formula as above but replacing $\Delta_2^{*}(f,D)$ by $L \Delta_2^{*}(f,D)$ (or by a direct estimate of $\Delta_2^{*}(g \circ f,D)$) preserves the $(\tau,\delta)$ A-BB certificate while trading a small amount of utility for improved certifiability. We give full formal details on this in \Cref{sec:shrinkage}.

\section{Experiments}
\label{sec:experiments}

\textbf{Benchmark.} We study \textit{Arena-Hard Auto}~\citep{li2024crowdsourceddatahighqualitybenchmarks}, a popular LLM-judged benchmark built from 500 challenging Chatbot Arena queries, intended to approximate Arena preferences. It reports high separability and strong correlation with human rankings, making it an excellent candidate for robust meta-analysis.

\textbf{Judges.} We evaluate four judge models: \textbf{GPT-4o-mini-0718},  \textbf{QwQ-32B}, \textbf{DeepSeek-R1-Distill-32B}, and \textbf{GPT-3.5-Turbo}~\citep{qwen2.5,openai2024gpt4ocard,deepseekai2025deepseekr1incentivizingreasoningcapability}.

\textbf{Combination Logic.} Context-aware sensitivities are derived for formatting ($S_{\mathrm{fmt}}$) and schematic adherence ($S_{\mathrm{sch}}$) by fitting their estimators and, when intrinsic jitter is available, calling $\operatorname{context\_adjusted\_rms}$ to fold that variance into each static estimate. Each component is lower-bounded at $10^{-3}$ before aggregation to avoid passing zero sensitivity. We consider the following combination strategies.

\begin{enumerate}
\item \textbf{Conservative (combined\_abb\_conservative):} $\hat{S}=\max{(S_{\mathrm{fmt}}, S_{\mathrm{psy}}, S_{\mathrm{sch}})}$
\item \textbf{RMS (combined\_abb\_rms):} $\hat{S}=\sqrt{(S_{\mathrm{fmt}}^{2}+S_{\mathrm{psy}}^{2}+S_{\mathrm{sch}}^{2})/3}$
\end{enumerate}

In our experiments, we utilize combined\_abb\_rms.

\noindent\textbf{Hyperparameters.} In the results shown in the main paper, we fix $\tau = 0.05$ and $\texttt{dim}=500$, the number of questions in the Arena-Hard-Auto benchmark.  We explore settings for $\delta$ ranging from $0.01$ to $0.05$.

\noindent\textbf{Inherent jitter.} We run the forward judgment pass five times and compute the RMS sensitivity across runs. This is used to measure the naturally occurring variance in judgment which is not biased by any external factors we wish to control.

\noindent\textbf{Formatting sensitivity.} It has been observed that large language models can be extremely sensitive to prompt formatting~\cite{sclar_quantifying_2023}. This long-standing bias is relatively straightforward to model via a simple generative mechanism that uses LLMs to produce highly similar semantic variants of model responses. We implement such a mechanism and measure the sensitivity in our judges, and find that it is generally fairly large, with considerable variance across judges.

\noindent\textbf{Schematic adherence.} This measures the degree to which an LLM judge's overall judgments can be explained as a function of its per-criteria (factor-wise) judgments~\citep{feuer2025judgmentnoisedesignfailures}. We fit polynomial regression models with interaction terms and define sensitivity as $S_{\text{sch}} = \sqrt{1 - R^2_{\text{schematic}}}$; see \cref{sec:schematic_adherence} for details.

\subsection{Controlling formatting sensitivity bias.}

In figure~\ref{fig:qwq_comparison}, we demonstrate the revised scores generated by BBE after accounting for observed formatting sensitivity in a QwQ-32B judge on Arena-Hard-Auto. The biased judge exhibits a troubling and common trend; high-performing models receive inflated scores while maintaining reasonable confidence intervals, resulting in an overall distribution which contains systematic biases.

\begin{figure*}[htbp]
\centering
\includegraphics[width=0.9\textwidth]{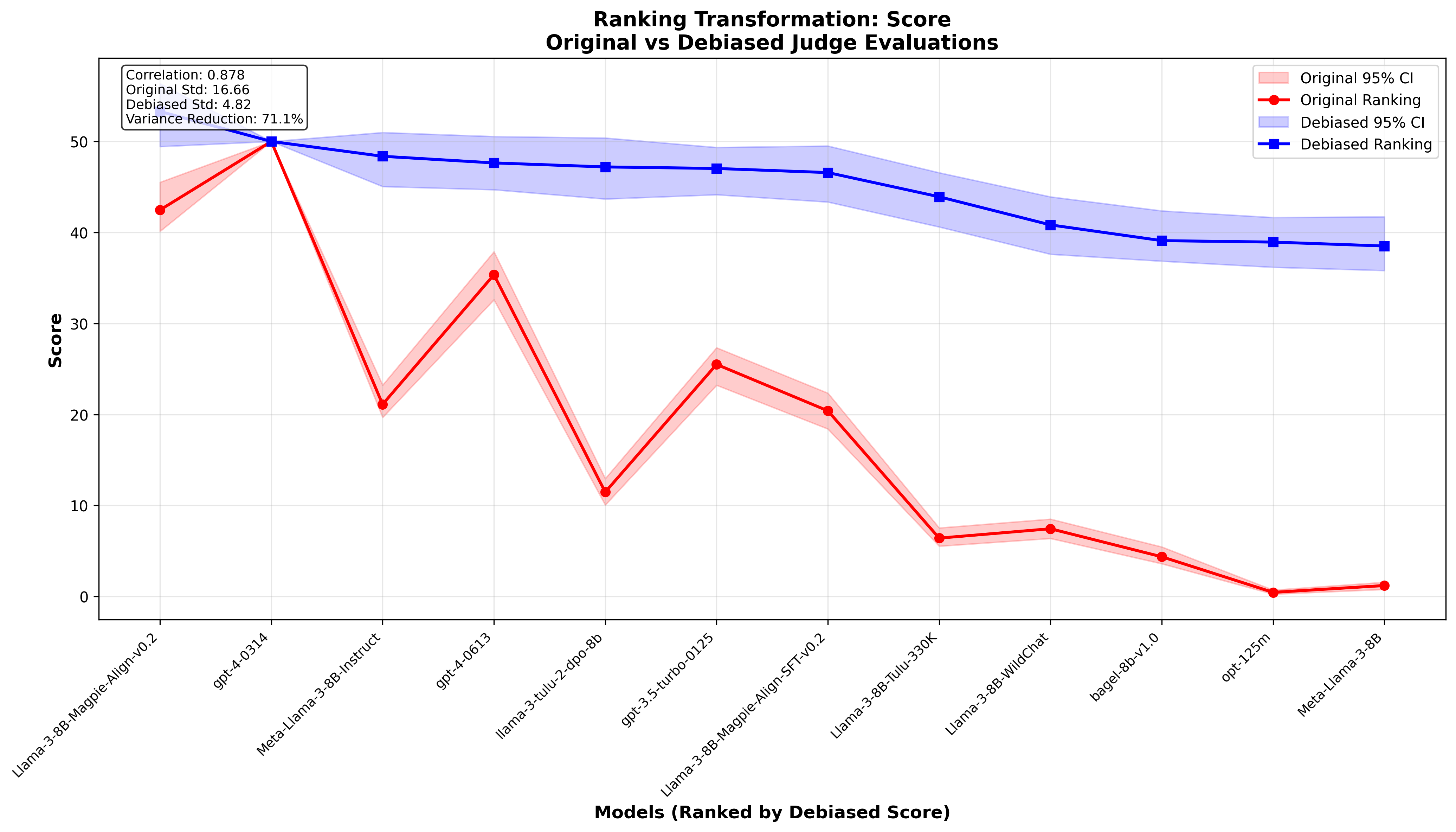}
\caption{\textbf{Bias-bounded transformation for formatting sensitivity.} The blue line in this figure corresponds to the debiased ranking generated by a QwQ-32B judge after using BBE with $\tau=0.5$ to control formatting sensitivity. Even with low $\tau$ tolerance, we are able to retain 88\% correlation with the original judgments in this realistic perturbation setting.}
\label{fig:qwq_comparison}
\end{figure*}

After BBE, we observe substantial variance reduction in the scores, indicating that we have successfully mitigated the judgment bias. At the same time, the debiased judgments retain strong signal, with an 81\% correlation to the original ranking. The apparent "certainty" of extreme judgments (e.g., facebook/opt-125m scoring 0.12 with CI width 0.24) is revealed as bias-induced false confidence, while the compressed scores represent genuine comparative signal with guarantees against failure.

\subsection{Controlling schematic bias.}

\begin{figure*}[htbp]
\centering
\includegraphics[width=0.9\textwidth]{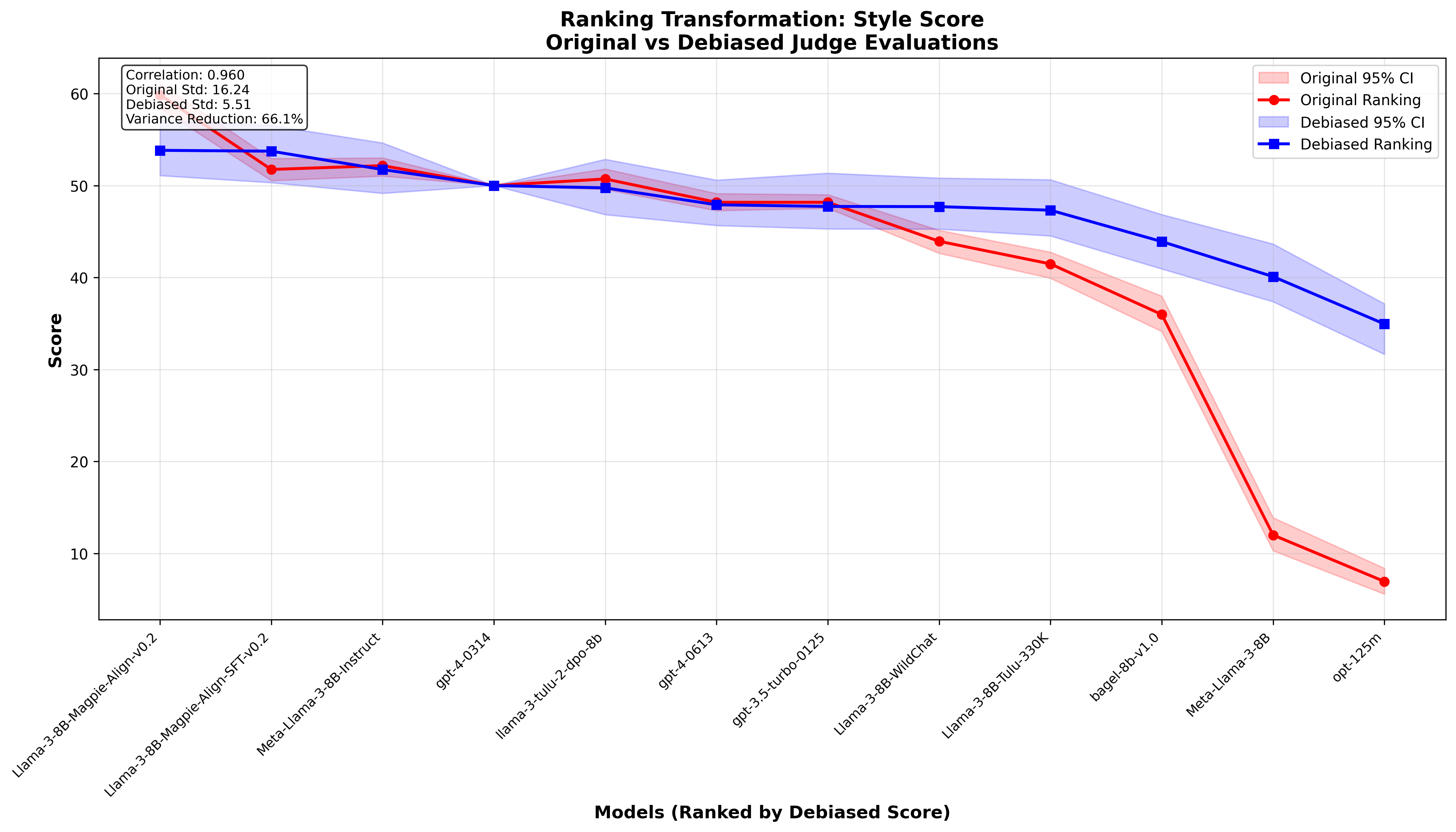}
\caption{\textbf{Bias-bounded evaluation in schematic sensitivity.} Even when measured bias is large, we are able to eliminate much potentially biased variance while retaining near-perfect correlation with the original judgments.($\tau=0.5$)}
\label{fig:gpt35_comparison}
\end{figure*}

Observed scores for schematic bias tend to be much larger than those of formatting bias, as they generally reflect structural weaknesses in the benchmark design rather than quasi-random failures of a particular judge. In \cref{fig:gpt35_comparison}, despite using a relatively old GPT 3.5 judge, the bias-bounded mechanism successfully compresses this extreme distribution into a realistic range while maintaining the ranking correlation nearly perfectly, further illustrating the mechanism's ability to distinguish between genuine performance differences and bias-induced score inflation. When using the more recent GPT 4o Mini as a judge, the correlation becomes nearly perfect.

\begin{figure*}[htbp]
\centering
\includegraphics[width=0.9\textwidth]{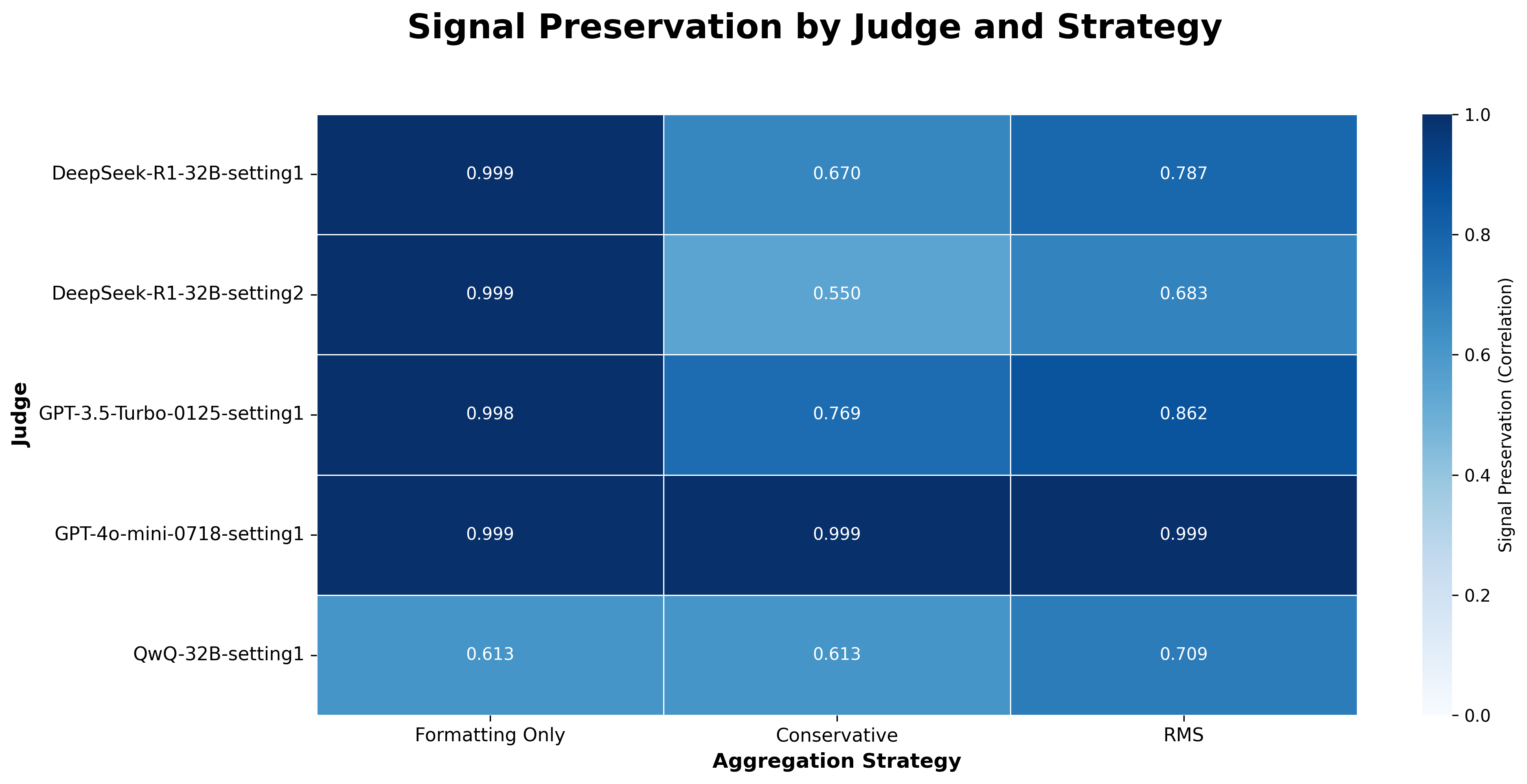}
\caption{\textbf{Correlative strength varies by judge and by dataset.} Although conservative aggregation strategies are always more difficult to debias, and simpler biases such as formatting are consistently easier to debias, other factors, such as the underlying dataset, can also have large effects.}
\label{fig:judge_comparison}
\end{figure*}

\section{Related Work}

Recent work has identified several critical bias issues in LLM judges. The CALM framework \citep{wang2024justice} quantifies 12 distinct bias types in LLM-as-judge systems, while research on ``preference leakage'' \citep{chen2024preference} demonstrates that LLMs exhibit systematic bias toward their own outputs and similar models used in training data. The ``Leaderboard Illusion'' \citep{clark2024leaderboard} provides evidence that prominent evaluation platforms like ChatBot Arena can be gamed by organizations with access to private testing. The ``Limits to Scalable Evaluation at the Frontier'' work \citep{li2024limits} shows that when judge accuracy equals evaluated model accuracy, debiasing methods can only reduce ground truth requirements by half, requiring bias to be small across all evaluated models. Research on individual preferences \citep{santurkar2024individual} demonstrates that different user groups systematically prefer different LLM responses, supporting the need for bias-bounded evaluation that can accommodate diverse judgment patterns.

Specialized benchmarks have emerged to evaluate judge quality, including JudgeBench \citep{zheng2024judgebench} for assessing LLM-based judges, LLMBar \citep{zhou2024llmbar} for instruction following evaluation, and JETTS \citep{wu2024jetts} for evaluating judges as evaluators. SafetyAnalyst \citep{mazeika2024safetyanalyst} represents specialized applications using LLM judgments with structured evaluation through ``harm-benefit trees'' and interpretable weight parameters for aggregating safety assessments.

Closest to our work is the ``Trust or Escalate'' framework \citep{pezeshkpour2024trust}, which provides rigorous guarantees of human agreement using cascaded LLM judges with calibrated confidence thresholds. Trust or Escalate uses ``Simulated Annotators''---ensemble agreement from multiple prompted perspectives---to estimate confidence, then abstains on low-confidence cases or escalates to stronger models. The guarantee is specifically: $P(\text{LLM agrees with human} \mid \text{LLM evaluates}) \geq 1-\alpha$. Crucially, this framework sidesteps bias by \textit{abstaining} when uncertain rather than bounding bias impact, requires human-labeled calibration data, and applies only to pairwise preference evaluation. In contrast, A-BB provides guarantees on \textit{all} evaluations (no abstention), handles adversarially discovered biases whose magnitude is bounded by measured sensitivities, requires no human labels, and extends to general scoring beyond pairwise comparison. We summarize these differences in \cref{tab:comparison}.

\begin{table}[t]
\caption{Comparison of our Average Bias-Boundedness (A-BB) with the best known existing framework for guaranteeing LLM judgments, Trust or Escalate (ToE).}
\label{tab:comparison}
\begin{center}
\begin{footnotesize}
\begin{tabular}{@{}lcc@{}}
\toprule
\textbf{Property} & \textbf{ToE} & \textbf{A-BB} \\
\midrule
Guarantees on all evaluations & \xmark & \cmark \\
Handles unknown biases$^{\dagger}$ & \xmark & \pmark \\
No human labels required & \xmark & \cmark \\
General scoring (beyond pairwise) & \xmark & \cmark \\
Bounds bias impact directly & \xmark & \cmark \\
\midrule
Human agreement guarantee$^{\ddagger}$ & \cmark & \pmark \\
Selective abstention & \cmark & \xmark \\
\bottomrule
\end{tabular}
\end{footnotesize}
\end{center}
\vskip -0.05in
{\scriptsize $^{\dagger}$A-BB bounds unknown biases only if their RMS sensitivity is bounded by that of measured biases.\\
$^{\ddagger}$A-BB can be combined with conformal prediction methods~\citep{ye2025conformal} to obtain human agreement guarantees.}
\vskip -0.1in
\end{table}

\paragraph{Formal guarantees and uncertainty quantification.} Recent work has begun providing statistical guarantees for LLM judge outputs. Conformal prediction has emerged as a distribution-free framework for quantifying judge uncertainty; \citet{ye2025conformal} apply split conformal prediction to construct prediction intervals with finite-sample coverage guarantees, ensuring that intervals contain the true judgment value at a user-specified rate. While conformal methods bound \textit{uncertainty} in individual judgments, our A-BB framework bounds the \textit{impact} of systematic biases across evaluation batches---the two approaches are complementary. On the calibration front, \citet{huang2025calibrating} propose extracting uncertainty signals from LLM internal representations via linear probes, addressing the well-documented overconfidence problem in LLM judges without requiring expensive fine-tuning or architectural modifications.

\paragraph{Scoring bias and agreeableness.} Beyond the bias types catalogued by CALM, recent work has identified additional failure modes. \citet{son2025scoring} define ``scoring bias'' as judgment shifts under superficial perturbations, identifying three specific variants: rubric order bias, score ID bias, and reference answer bias---all of which manifest as sensitivity to factors orthogonal to response quality. Our neighbor generator framework naturally captures such perturbations. Separately, research on agreeableness bias \citep{liu2025consensus} reveals that LLM judges exhibit True Negative Rates below 25\%, indicating systematic over-agreement with evaluated outputs. This tendency to validate rather than critique represents a bias pattern that A-BB can help bound, as it manifests as measurable sensitivity to the presence of errors in evaluated content.

\section{Limitations}

Our bias-bounded framework addresses measurable systematic bias patterns but makes no claims about absolute accuracy, whether scores are calibrated well across multiple judges, or whether judges are actually judging all aspects of their input. The guarantee provided by our framework also can be misinterpreted if the measures used to calibrate bias are not reliable (e.g., if they miss a bias source of greater magnitude than any that is considered). There is, however, considerable ongoing work in these areas, and we see our efforts as complementary to those. Last but not least, the preferences of an LLM judge are themselves a bias pattern (in the sense that they are systematic deviations from random noise); A-BB is not a "magic bullet", and will be much more effective in retaining true signal when the intrinsic jitter and systematic bias of the underlying judge is weaker.

\paragraph{Finite-sample estimation.} Our theoretical guarantees (\Cref{thm:gauss_mech_abb,cor:split_failure_budget}) assume access to the true root-mean-squared sensitivity $\Delta_2^*(f,D)$, which is defined as an expectation over the neighbor generator $T$. In practice, \Cref{alg:gauss_mech} estimates this quantity empirically using $m$ sampled neighbors. With finite $m$, there is a non-zero probability that the empirical estimate underestimates the true sensitivity, which could cause the actual failure probability to exceed $\delta$. Practitioners should mitigate this by using sufficiently large $m$, or by adding a confidence margin (e.g., using an upper confidence bound on the sensitivity estimate). Formally incorporating this estimation uncertainty into the $\delta$ budget via concentration inequalities (e.g., Chernoff bounds) is a direction for future work.

\section{Conclusion}
\label{sec:conclusion}

We have presented a formal framework for providing bias-bounded guarantees in LLM judge mechanisms. Broadly scoped formal guarantees, wedded to the power of LLM judges, enable use cases previously considered challenging, such as conducting social sciences research with the aid of LLMs, or determining a candidate's suitability for a loan. Rather than attempting to enumerate and eliminate every possible bias source on a case by case basis, our insight is that by guaranteeing \textit{any} bias pattern of sufficient magnitude will be  indistinguishable from noise, we can enable greater confidence in LLM judges. The framework is general enough to apply to various LLM evaluation scenarios while providing the mathematical rigor needed for formal bias bounds.

\section*{Reproducibility Statement}

We have, to the best of our ability, ensured that all experiments described in this paper are reproducible in principle. In order to facilitate this, we provide source code containing everything needed to reproduce our experiments at \url{https://github.com/penfever/bias-bounded-evaluation}.

\section*{Impact Statement}

This paper presents work whose goal is to advance the field of Machine Learning, specifically in providing formal guarantees for LLM-based evaluation systems. Our framework for bias-bounded evaluation aims to make AI systems more reliable and trustworthy by quantifying and mitigating systematic biases in LLM judges. While we believe this work has positive implications for the safe deployment of autonomous AI systems, we acknowledge that any evaluation framework can potentially be misused. We encourage practitioners to use this framework as part of a broader evaluation strategy that includes human oversight.

\bibliographystyle{plainnat}
\bibliography{bbe}

\newpage
\appendix

\section{Deferred Proofs}
\label{sec:deferred_proofs}

\GaussMechABB*

\begin{proof}
    Note that $M_\sigma(D) - M_\sigma(D') = \Delta + B$. Thus, for any $a \in (0,\tau)$, the event split,
    \begin{align*}
    \{\|\Delta + B\|_2 > \tau\} \subseteq \{\|\Delta\|_2 > a\} \cup \{\|B\|_2 > \tau - a\}~,
    \end{align*}
    holds via the triangle inequality. So, we can take probabilities of the events and use a union bound to get,
    \begin{align*}
    Pr\left[\| \Delta + B\|_2 > \tau \right] \leq Pr\left[\|\Delta\|_2 > a\right] + Pr\left[\|B\|_2 > \tau - a\right]~.
    \end{align*}
    Consider the $B$ term. Since $B \sim \mathcal N(0, 2\sigma^2 I_d)$, we have $\|B\|_2^2 / (2 \sigma^2) \sim \chi^2_d$. We can thus apply the Laurent–Massart inequality for a chi-squared r.v. with $d$-degrees of freedom \cite{laurent2000adaptive}, which gives that for any $x>0$,
    \begin{align*}
    Pr\left(\|B\|_2 > \sqrt{2 \sigma^2 \left(d + 2\sqrt{dx} + 2x \right)} \right) \leq e^{-x}~.
    \end{align*}
    If we set $x = \log(1 / \delta_B)$, we then get,
    \begin{align*}
    Pr\left[\|B\|_2 > \sqrt{2 \sigma^2(d + 2\sqrt{d \log(1 / \delta_B)} + 2\log(1  / \delta_B))}\right] \leq \delta_B~.
    \end{align*}
    So, letting $\sqrt{2 \sigma^2(d + 2\sqrt{d \log(1 / \delta_B)} + 2\log(1  / \delta_B))} \leq \tau-a$ gives $Pr[\|B\|_2 > \tau - a] \leq \delta_B$.

    Now weconsider the $\Delta$-term. By Markov, we can consider the nonnegative random variable $\|\Delta\|_2^2$ (randomness here is only over $D' \underset{T}{\sim} D$). We get,
    \begin{align*}
    Pr[\|\Delta\|_2>a] = Pr[\|\Delta\|_2^2>a^2] \le \frac{\mathbb{E}\|\Delta\|_2^2}{a^2} = \frac{\Delta_2^{*}(f, D)^2}{a^2}~.
    \end{align*}
    Combining the two bounds completes the proof.
\end{proof}

\SplitFailureBudget*

\begin{proof}
Apply Theorem~\ref{thm:gauss_mech_abb} with $a = \Delta^*_2(f,D) / \sqrt{\delta_\Delta}$.
Then $Pr[\|\Delta\|_2 > a] \leq \delta_\Delta$ by Markov, and choosing
$\sigma \leq \sigma_{\max}$ makes $\Pr[\|B\|_2 > \tau - a] \leq \delta_B$ by Laurent–Massart.
Union bounding then yields $\delta_B + \delta_\Delta$.
\end{proof}

\WhySplittingHelps*

\begin{proof}
This is an algebraic fact. Since $\Delta \in [0,\tau)$, both $\tau - \Delta$ and $\tau + \Delta$ are positive. Then, we expand,
\begin{align*}
\sqrt{\tau^2 - \Delta^2} = \sqrt{(\tau - \Delta)(\tau + \Delta)}~.
\end{align*}
We divide through by $\sqrt{\tau + \Delta}$, which gives
\begin{align*}
\frac{\tau - \Delta}{\sqrt{\tau^2 - \Delta^2}} = \frac{\tau - \Delta}{\sqrt{(\tau - \Delta)(\tau + \Delta)}} = \sqrt{\frac{\tau - \Delta}{\tau + \Delta}}~.
\end{align*}
We know that $\tau + \Delta > \tau - \Delta \geq 0$, thus the fraction under the square root lies in $(0, 1)$ and so its square root is strictly less than $1$ as well.
\end{proof}

\Utility*
\begin{proof}
This is a standard fact of Gaussian noise. By definition $M_\sigma(D) = f(D) + Z$ with $Z \sim \mathcal N(0, \sigma^2 I_d)$. So we have,
\begin{align*}
\|M_\sigma(D) - f(D)\|_2^2 = \|Z\|_2^2 = \sum_{i=1}^d Z_i^2~,
\end{align*}
where each $Z_i$ is i.i.d. $\mathcal N(0, \sigma^2)$. Therefore,
\begin{align*}
\mathbb{E}[\|M_\sigma(D) - f(D)\|_2^2] = \sum_{i=1}^d \mathbb{E}[Z_i^2] = \sum_{i=1}^d Var(Z_i) = d\sigma^2~.
\end{align*}
\end{proof}




\section{Shrinkage Mechanism}
\label{sec:shrinkage}

\paragraph{Setup.} Let $f:\mathcal{D}\to\R^d$ be a (deterministic) scoring functional and let $T$ denote a neighbor generator that induces a distribution over neighbors $D'\sim T(D)$. Recall the RMS average sensitivity
\begin{align*}
\Delta^*_2(f,D) := \left(\E_{D'\sim T(D)}\left[\| f(D) - f(D')\|_2^2\right]\right)^{1/2},
\end{align*}
and the $(\tau,\delta)$ A-BB requirement $\Pr_{D' \sim T(D)}\left[ \|M(D) - M(D')\|_2 > \tau \right]\le \delta$.

\paragraph{Deterministic Lipschitz post-processing.} We consider composing $f$ with a deterministic $L$-Lipschitz map $g: \R^d \to \R^d$ (e.g., affine shrinkage $g(x) = \alpha x + (1 - \alpha) \mu$ with $L = \alpha \in (0,1]$). The following contraction of RMS sensitivity is immediate.

\begin{lemma}[RMS sensitivity contracts under Lipschitz maps]\label{lem:lipschitz}
For any $L$-Lipschitz $g$ and any $f,D$,
\begin{align*}
\Delta^*_2(g\circ f, D) \leq L\Delta^*_2(f,D).
\end{align*}
\end{lemma}

\begin{proof}
For every realization $D'\sim T(D)$, Lipschitzness gives
$\|g(f(D)) - g(f(D'))\|_2 \leq L\|f(D) - f(D')\|_2$. After squaring, we take expectations over $D'$, and then take square roots, to yield the result.
\end{proof}

\begin{proposition}[A-BB with Lipschitz shrinkage]\label{prop:shrink}
Fix $\tau > 0$ and split $\delta = \delta_B + \delta_\Delta$ with $\delta_B, \delta_\Delta \in (0,1)$. Consider
\begin{align*}
M_\sigma(D)= g(f(D)) + Z,\qquad Z \sim \mathcal{N}(0,\sigma^2 I_d),
\end{align*}
with $g$ $L$-Lipschitz. If
\begin{align*}
\tau > L\Delta^*_2(f,D)\sqrt{\tfrac{1}{\delta_\Delta}},
\end{align*}
then $(\tau,\delta)$-A-BB holds for any $0 < \sigma \leq \sigma_{\max}$ where
\begin{align*}
\sigma_{\max} = \frac{\tau - L\Delta^*_2(f,D)\sqrt{1 / \delta_\Delta}}{\sqrt{2( d + 2\sqrt{d\log(1 / \delta_B)} + 2\log(1 / \delta_B))}}.
\end{align*}
Under the symmetric split $\delta_B=\delta_\Delta=\delta/2$ this specializes to the replacement $\Delta_2^{*}(f,D)\mapsto L\Delta_2^{*}(f,D)$ in \Cref{cor:split_failure_budget}.
\end{proposition}

\begin{proof}
Apply \Cref{thm:gauss_mech_abb} to $g\circ f$ and use \Cref{lem:lipschitz} to upper bound $\Delta_2^{*}(g\circ f,D)$ by $L\Delta_2^{*}(f,D)$.
\end{proof}

The concentration bound follows from well-established theory on the tail behavior of Gaussian random vectors in high dimensions~\cite{vershynin2018high}.

\paragraph{Lipschitz contractions (shrinkage).} Consider a deterministic map $g: \R^d\to\R^d$ that is $L$-Lipschitz:
\[
\|g(x)-g(y)\|_2 \;\le\; L\,\|x-y\|_2\quad \text{for all }x,y\in\R^d.
\]
An important special case is affine shrinkage around a center $\mu\in\R^d$ with coefficient $\alpha\in(0,1]$:
\[
g(x) \;=\; \alpha\,x + (1-\alpha)\,\mu,\qquad L=\alpha.
\]
This form of shrinkage is closely related to classical shrinkage estimators in statistics~\cite{james1961estimation,efron1973stein} and regularization techniques in machine learning~\cite{chaudhuri2011differentially}.

\begin{lemma}[RMS sensitivity contracts under Lipschitz maps]\label{lem:lipschitz2}
For any $L$-Lipschitz $g$ and any $f,D$, we have $\Delta^*_2(g\circ f, D) \le L\,\Delta^*_2(f,D)$.
\end{lemma}

\begin{proof}
By Lipschitzness, for every realization $D'\sim T(D)$,
\[
\|g(f(D)) - g(f(D'))\|_2 \;\le\; L\,\|f(D)-f(D')\|_2.
\]
Squaring and taking expectations over $D'\sim T(D)$ yields
\[
\E\big[\,\|g(f(D)) - g(f(D'))\|_2^2\,\big] \;\le\; L^2\,\E\big[\,\|f(D)-f(D')\|_2^2\,\big].
\]
Taking square roots gives the claim.
\end{proof}

\begin{proposition}[Shrinkage A-BB via Lipschitz contraction]\label{prop:shrink2}
Fix $\tau>0$ and a split $\delta=\delta_B+\delta_\Delta$ with $\delta_B,\delta_\Delta\in(0,1)$. Let $g$ be $L$-Lipschitz and consider the Gaussian mechanism $M_\sigma(D)= g(f(D)) + Z$ with $Z\sim\mathcal{N}(0,\sigma^2 I_d)$. If
\[
\tau \;>\; L\,\Delta^*_2(f,D)\,\sqrt{\tfrac{1}{\delta_\Delta}},
\]
then $(\tau,\delta)$-A-BB holds for
\[
\sigma \;=\; \frac{\;\tau - L\,\Delta^*_2(f,D)\,\sqrt{1/\delta_\Delta}\;}{\sqrt{\,2\,\big( d + 2\sqrt{d\,\log(1/\delta_B)} + 2\,\log(1/\delta_B)\big)\,}}\,.
\]
In particular, for shrinkage $g(x)=\alpha x + (1-\alpha)\mu$ with $L=\alpha$, the precondition becomes $\tau>\alpha\,\Delta^*_2(f,D)\,\sqrt{1/\delta_\Delta}$, which is feasible for any target $\tau>0$ by choosing $\alpha\le \tau/(\Delta^*_2(f,D)\,\sqrt{1/\delta_\Delta})$. Under the symmetric split $\delta_B = \delta_\Delta = \delta/2$, this specializes to $\tau > \alpha \Delta^*_2(f,D)\sqrt{2 / \delta}$.
\end{proposition}

\begin{proof}
By Lemma~\ref{lem:lipschitz2}, $\Delta^*_2(g\circ f, D)\le L\,\Delta^*_2(f,D)$. Apply the Gaussian A-BB bound (Proposition~1 in the main text; cf. \cite{dwork2006calibrating,balle2018improving}) to $g\circ f$ with sensitivity upper bound $L\,\Delta^*_2(f,D)$, the same $(\tau,\delta_B,\delta_\Delta)$, and obtain the stated $\sigma$.
\end{proof}

\paragraph{Remarks.}
\begin{itemize}
  \item Shrinkage trades utility for certifiability: smaller $\alpha$ reduces the effective sensitivity and allows certifying smaller $\tau$; it also attenuates variation in outputs.
  \item Non-expansive projections $g$ (e.g., L2 projection onto a ball~\cite{bauschke2017convex}) have $L\le 1$ and can deterministically bound output diameters; one can combine such $g$ with small $\sigma$ if randomized outputs are desired.
  \item The same reasoning extends to other noise families provided appropriate tail bounds; the Gaussian case is convenient due to spherical symmetry and well-known concentration~\cite{dong2022gaussian}.
\end{itemize}

\paragraph{Choosing the shrinkage center $\mu$.} The choice of $\mu$ affects both utility and the certified bound.
If $\mu$ is fixed (data-independent), Lemma~\ref{lem:lipschitz2} applies directly. If $\mu$ depends on the dataset, define $g_D(x)=\alpha x + (1-\alpha)\mu(D)$. For one-step neighbors $D,D'$, we have
\[
\|g_D(f(D)) - g_{D'}(f(D'))\|_2 \;\le\; \alpha\,\|f(D)-f(D')\|_2 + (1-\alpha)\,\sqrt{d}\,|\mu(D)-\mu(D')|.
\]
Taking RMS over $D'\sim T(D)$ yields an effective sensitivity bound
\[
\Delta^{\text{eff}}_2 \;\le\; \alpha\,\Delta^*_2(f,D) + (1-\alpha)\,S_\mu,\qquad S_\mu := \sqrt{d}\cdot \E\big[\,|\mu(D)-\mu(D')|\,\big].
\]
For the sample mean on $[0,1]$, $|\mu(D)-\mu(D')| \le 1/n$ for $n$ samples, hence $S_\mu \le \sqrt{d}/n$. Thus the precondition with adaptive $\mu$ becomes
\[
\tau \;>\; \big(\alpha\,\Delta^*_2(f,D) + (1-\alpha)\,S_\mu\big)\,\sqrt{\tfrac{1}{\delta_\Delta}}.
\]
Solving for $\alpha$ in normalized units (letting $A=\Delta^*_2\sqrt{1/\delta_\Delta}$ and $B=S_\mu\sqrt{1/\delta_\Delta}$) gives $\alpha < (\tau - B)/(A - B)$ provided $A>B$. In practice, for large $n$ and $d\approx n$, $B\approx 1/\sqrt{n}$ is small. Data-driven yet \emph{fixed-within-batch} centers (e.g., holdout means, exponential moving averages frozen per batch, or profile means) recover the $S_\mu=0$ case since $\mu$ does not change under the neighbor in that batch.

\paragraph{Center strategies.} Suitable choices include: (i) fixed/public centers (profile means) for $S_\mu=0$, (ii) holdout means computed on a disjoint subset (unchanged by a one-item neighbor in the evaluation subset), and (iii) EMA centers updated between batches but held fixed during certification~\cite{hazan2016introduction}. Robust centers (median/trimmed means~\cite{huber2009robust}) remain admissible with the same form and a typically small $S_\mu$. For high-dimensional covariance shrinkage, similar principles apply with appropriate modifications~\cite{ledoit2004well}.

\paragraph{Coverage of unmodeled biases and conservative calibration.}
Let $\mathcal{S}$ be a family of neighbor generators (potential biases). Suppose we calibrate a single bound
\[
 B \;\ge\; \sup_{S\in\mathcal{S}} \Delta^*_2\big(f,D; S\big).
\]
Using Proposition~\ref{prop:shrink2} with $\Delta^*_2(f,D)$ replaced by $B$ yields an $(\tau,\delta)$-A-BB certificate for \emph{every} $S\in\mathcal{S}$ as soon as $\tau > L\,B\,\sqrt{1/\delta_\Delta}$, with the same Gaussian calibration of $\sigma$. In particular, if we have a finite set of measured generators $\{T_i\}_{i=1}^m$, then choosing
\[
 B \;=\; \max_{i\in[m]} \Delta^*_2\big(f,D; T_i\big)
\]
certifies against the worst measured bias and any $S\in\mathcal{S}$ whose RMS sensitivity is no larger than this envelope. This includes any convex mixture $\sum_i w_i T_i$ because
\[
 \Delta^*_2\!\left(f,D;\sum_i w_i T_i\right)^2
 \;=\; \sum_i w_i\, \E_{D'\sim T_i(D)}\big[\,\| f(D) - f(D')\|_2^2\,\big]
 \;\le\; \max_i \Delta^*_2\big(f,D;T_i\big)^2.
\]
We refer to this choice $B = \max_i \Delta^*_2(f,D;T_i)$ as the \emph{conservative} calibration: it trades more shrinkage/noise for coverage of all measured biases and any dominated (``no larger'') unknown bias under the same RMS metric. This approach is related to multiple testing corrections~\cite{benjamini1995controlling} and privacy composition theory~\cite{kairouz2015composition}.

If the shrinkage center $\mu$ depends on the dataset, replace $B$ by the effective bound
\[
 B_{\mathrm{eff}} \;=\; \alpha\,B + (1-\alpha)\,S_\mu,
\]
and require $\tau > B_{\mathrm{eff}}\,\sqrt{1/\delta_\Delta}$ (cf. the discussion above). Using a center that is fixed within the certified batch (public, holdout, or frozen EMA) recovers $S_\mu = 0$ and the simpler conservative condition.

\section{Schematic Adherence}
\label{sec:schematic_adherence}

Schematic adherence measures the degree to which an LLM judge's overall judgments can be explained as a function of its per-criteria (factor-wise) judgments~\citep{feuer2025judgmentnoisedesignfailures}. Low schematic adherence indicates that the judge's verdicts deviate from the stated rubric in ways that cannot be captured by any weighting of the explicit factors.

\paragraph{Measurement.} Given a judge that produces factor scores $f_{i1}, \ldots, f_{ik}$ and an overall verdict $o_i$ for each sample $i$, we fit two regression models:

\textbf{Linear model:}
\[
o_i = \beta_0 + \sum_{j=1}^{k} \beta_j f_{ij} + \varepsilon_i
\]

\textbf{Polynomial model (with interactions):}
\[
o_i = \beta_0 + \sum_{j=1}^{k} \beta_j f_{ij} + \sum_{j=1}^{k} \beta_{jj} f_{ij}^2 + \sum_{j < \ell} \beta_{j\ell} f_{ij} f_{i\ell} + \varepsilon_i
\]

The schematic adherence score is defined as:
\[
R^2_{\text{schematic}} = \max(R^2_{\text{linear}}, R^2_{\text{polynomial}})
\]

\paragraph{Sensitivity metric.} For bias-bounded evaluation, we convert schematic adherence into a sensitivity measure:
\[
S_{\text{sch}} = \sqrt{1 - R^2_{\text{schematic}}}
\]

This quantifies the fraction of judgment variance attributable to factors outside the explicit rubric---precisely the ``implicit bias'' that A-BB aims to bound. Empirically, $R^2_{\text{schematic}}$ ranges from 0.10 (DeepSeek-R1-32B) to 0.74 (GPT-4o-mini), indicating substantial variation in how faithfully different judges follow their stated evaluation criteria~\citep{feuer2025judgmentnoisedesignfailures}.

\end{document}